\documentclass[11pt]{article}

\usepackage{tikz}
\usetikzlibrary{intersections}
\usepackage{multicol}
\usetikzlibrary{positioning}
\usetikzlibrary{decorations.pathreplacing}
 \usetikzlibrary{calc} 
\pgfdeclarelayer{background}
\pgfsetlayers{background,main}

\usepackage[most]{tcolorbox}

\newtcolorbox{redbox}{
    arc=0pt,
    boxrule=0pt,
    colback=red,
    width=.8\textwidth,   
    colupper=black,
    fontupper=\bfseries
}

\usepackage{algpseudocode}

\usepackage[utf8]{inputenc}
\usepackage[colorlinks = true,
            linkcolor = blue,
            urlcolor  = blue,
            citecolor = blue,
            anchorcolor = blue]{hyperref}
\usepackage{amsmath}
\usepackage{amsthm}
\usepackage{amssymb}
\usepackage{xcolor}
\usepackage{algorithm,caption}
\usepackage{cleveref}
\usepackage{graphicx}
\usepackage{scalerel}
\usepackage{enumitem}
\usepackage{natbib}
\usepackage{framed}
\usepackage[most]{tcolorbox}

\definecolor{gray}{gray}{0.4}

\definecolor{ama-iro}{RGB}{0, 158, 243.0}
\definecolor{fuyu-gaki}{RGB}{251, 74, 52}
\definecolor{momiji}{RGB}{245, 70, 111}
\definecolor{hotaru-bi}{RGB}{229,221,58} 
\definecolor{kon-peki}{RGB}{1,120,217}
\definecolor{shin-kai}{RGB}{77,98,152}
\definecolor{shin-ryoku}{RGB}{1,145,97}
\definecolor{yama-budo}{RGB}{171,14,122}

\definecolor{citecolor}{HTML}{5E9100}

\definecolor{theoremcolor}{HTML}{FFE1D9}
\definecolor{resultcolor}{HTML}{FCDFBE}

\definecolor{constraintcolor}{HTML}{BBD49B}
\definecolor{goalcolor}{HTML}{CAE4A7}

\definecolor{remarkcolor}{HTML}{E3EEC7}

\definecolor{definitioncolor}{HTML}{FFFDDA}
\definecolor{examplecolor}{HTML}{F9E9D9}
\definecolor{questioncolor}{HTML}{DDE8EB}
\definecolor{conjecturecolor}{HTML}{DDE8EB}

\definecolor{lightblue}{HTML}{C3E1FF}

\definecolor{captioncolor}{RGB}{128,128,128}

\definecolor{captioncolor}{RGB}{128,128,128}

\theoremstyle{plain}
\newtheorem{theorem}{Theorem}[section]
\newtheorem{lemma}[theorem]{Lemma}

\newtheorem{prop}[theorem]{Proposition}

\newtheorem{corollary}[theorem]{Corollary}
\newtheorem{question}[theorem]{Open Problem}
\newtheorem{question*}{Question}
\newtheorem{conjecture}[theorem]{Conjecture}

\newtheorem{observation}[theorem]{Observation}
\newtheorem{assumption}[theorem]{Assumption}

\crefname{claim}{Claim}{Claims}
\crefname{fact}{Fact}{Facts}

\theoremstyle{definition}

\newtheorem{definition}[theorem]{Definition}
\newtheorem{remark}[theorem]{Remark}
\newtheorem{example}[theorem]{Example}
\newtheorem{goal}[theorem]{Goal}
\newtheorem{goal*}{Goal}

\theoremstyle{plain}

\newenvironment{rlemma}[1]{\renewcommand\thelemma{\ref*{#1}}\begin{lemma}[Restated]}{\end{lemma}}

\newenvironment{rdefinition}[1]{\renewcommand\thedefinition{\ref*{#1}}\begin{definition}[Restated]}{\end{definition}}
\tcbuselibrary{theorems}

\newtcbtheorem[use counter=theorem, number within=section]{boxtheorem}{Theorem}%
{colback=theoremcolor,
 colframe=black,
 coltitle=black,
 colbacktitle=theoremcolor,
 fonttitle=\bfseries,
 separator sign=:,
 enhanced,
 boxed title style={size=small}}{Thm}

\newtcbtheorem[use counter=theorem, number within=section]{boxdef}{Definition}%
{colback=definitioncolor,
 colframe=black,
 coltitle=black,
 colbacktitle=definitioncolor,
 fonttitle=\bfseries,
 separator sign=:,
 enhanced,
 boxed title style={size=small}}{Def}

\tcolorboxenvironment{boxlemma}{colback=resultcolor, colframe=black,
    colbacktitle=resultcolor, coltitle=resultcolor}

\tcolorboxenvironment{rboxlemma}{colback=resultcolor, colframe=black,
	colbacktitle=resultcolor, coltitle=resultcolor}

\tcolorboxenvironment{boxproposition}{colback=resultcolor, colframe=black,
    colbacktitle=resultcolor, coltitle=resultcolor}

\tcolorboxenvironment{boxcorollary}{colback=resultcolor, colframe=black,
    colbacktitle=resultcolor, coltitle=resultcolor}

\tcolorboxenvironment{boxobservation}{colback=resultcolor, colframe=black,
    colbacktitle=resultcolor, coltitle=resultcolor}

\tcolorboxenvironment{boxquestion}{colback=questioncolor, colframe=black,
    colbacktitle=questioncolor, coltitle=shin-kai}
\newenvironment{boxquestion*}{\begin{question*}}{\end{question*}}
\tcolorboxenvironment{boxquestion*}{colback=questioncolor, colframe=black,
    colbacktitle=questioncolor, coltitle=questioncolor}

\tcolorboxenvironment{boxdefinition}{colback=definitioncolor, colframe=black,
    colbacktitle=definitioncolor, coltitle=definitioncolor}

\tcolorboxenvironment{rboxdefinition}{colback=definitioncolor, colframe=black,
	colbacktitle=definitioncolor, coltitle=definitioncolor}

\tcolorboxenvironment{boxassumption}{colback=definitioncolor, colframe=black,
    colbacktitle=definitioncolor, coltitle=definitioncolor}

\tcolorboxenvironment{boxexample}{colback=examplecolor, colframe=black,
    colbacktitle=examplecolor, coltitle=examplecolor}

\newtheorem{exercise}[theorem]{Exercise}

\tcolorboxenvironment{boxexercise}{colback=exercisecolor, colframe=black,
    colbacktitle=exercisecolor, coltitle=exercisecolor}

\tcolorboxenvironment{boxgoal}{colback=goalcolor, colframe=black,
    colbacktitle=goalcolor, coltitle=goalcolor}

\newenvironment{boxgoal*}{\begin{goal*}}{\end{goal*}}
\tcolorboxenvironment{boxgoal*}{colback=goalcolor, colframe=black,
    colbacktitle=goalcolor, coltitle=goalcolor}

\tcolorboxenvironment{boxremark}{colback=remarkcolor, colframe=black,
    colbacktitle=remarkcolor, coltitle=remarkcolor}

\tcolorboxenvironment{boxconjecture}{colback=conjecturecolor, colframe=black,
    colbacktitle=conjecturecolor, coltitle=shin-kai}
\newenvironment{boxconjecture*}{\begin{conjecture*}}{\end{conjecture*}}
\tcolorboxenvironment{boxconjecture*}{colback=conjecturecolor, colframe=black,
    colbacktitle=conjecturecolor, coltitle=conjecturecolor}

\usepackage{geometry}
\newtheorem{innercustomthm}{Theorem}
\newenvironment{customthm}[1]
  {\renewcommand\theinnercustomthm{#1}\innercustomthm}
  {\endinnercustomthm}

\newcommand{\vc}{\mathtt{VC}}
\newcommand{\R}{\mathbb{R}}

\newcommand{\X}{\mathcal{X}}

\renewcommand{\H}{\mathcal{H}}
\newcommand{\F}{\mathcal{F}}
\newcommand{\B}{\mathcal{B}}

\newcommand{\A}{\mathcal{A}}
\newcommand{\eps}{\epsilon}

\newcommand{\sign}{\mathsf{sign}}

\newcommand{\D}{\mathcal{D}}

\newcommand{\N}{\mathbb{N}}

\renewcommand{\Pr}{\mathop{\mathbb{P}}}

\newcommand{\Dx}{\mathcal{D}_{\X}}

\newcommand{\fat}{\mathtt{fat}}

\newcommand{\Lip}{\mathtt{Lip}}

\newcommand{\floor}[1]{\lfloor#1\rfloor}

\newcommand{\norm}[1]{\left\lVert#1\right\rVert}

        \hypersetup{
           breaklinks=true,   
           colorlinks=true,   
        }

\author{
  Yair Ashlagi\thanks{School of Electrical and Computer Engineering, Tel Aviv University.}
  \and 
  Roi Livni\thanks{School of Electrical and Computer Engineering, Tel Aviv University.}
  \and  
  Shay Moran\thanks{Departments of Mathematics, Computer Science, and Data and Decision Sciences, Technion, and Google Research.}  
  \and 
  Tom Waknine\thanks{Departments of Mathematics, Technion.}
}

\begin{document}
\title{Margin in Abstract Spaces}
\maketitle
\begin{abstract}
Margin-based learning, exemplified by linear and kernel methods, is one of the few classical settings where generalization guarantees are independent of the number of parameters. This makes it a central case study in modern highly over-parameterized learning. We ask what minimal mathematical structure underlies this phenomenon.

We begin with a simple margin-based problem in arbitrary metric spaces: concepts are defined by a center point and classify points according to whether their distance lies below $r$ or above $R$. We show that whenever $R>3r$, this class is learnable in \emph{any} metric space. Thus, sufficiently large margins make learnability rely only on the triangle inequality, without any linear or analytic structure being necessary. Our first main result extends this phenomenon to concepts defined by bounded linear combinations of distance functions, and reveals a sharp threshold: there exists a universal constant such that whenever the margin is larger than this constant, the class is learnable in every metric space, while below it there exist metric spaces where it is not learnable at all.

We then ask whether margin-based learnability can always be explained via an embedding into a linear space -- that is, reduced to linear classification in some Banach space through a kernel-type construction. We answer this negatively by demonstrating a margin learnable class that cannot be embedded into any Banach space in which linear classification with margins is learnable.

\end{abstract}

\section{Introduction}
%

Margin-based learning provides a canonical example of parametric learning in high dimensions due to it being one of the primary settings where generalization does not depend on the number of parameters. Certain hypothesis classes, most notably linear functions, whose learning complexity typically grows with dimensionality, become markedly easier to learn when their effective domain is restricted to points lying a fixed distance away from the decision boundary, i.e., when a margin condition is imposed.
Due to the appealing nature of dimensionality-independent generalization, margin assumptions have been a topic for extensive research in the context of halfspaces and kernel methods in Euclidean and Hilbert spaces. Classic maximal margin algorithms, such as the Support Vector Machines, enjoy dimension-independent generalization guarantees when the data satisfies a margin condition. Namely, when some target function realizes the data with a sufficiently large margin \cite{Vapnik1998, Cristianini2000, Bartlett2002}. 

However, much of the existing work relies on strong geometric assumptions, typically those of Euclidean or, more generally, Hilbert spaces (via kernel methods). This naturally raises the question: 
\begin{center}
\textbf{Which basic mathematical properties underlie margin-based learnability, and when can it be reduced to learning embeddings in linear spaces?}
\end{center}

We begin by abstracting linear margin-based hypotheses to the minimal geometric setting of metric spaces, and ask whether such weak structure suffices for learnability. Let $\X$ be an arbitrary metric space with distance function $d(\cdot,\cdot)$. A simple margin-based concept class in this setting is defined by points of the space: for a given center $x \in \X$ and parameters $0 \le r < R$, the corresponding concept labels points within distance at most $r$ from $x$ as positive, and points at distance greater than $R$ from $x$ as negative. Formally, we define
\begin{equation}
d_x(x') =
\begin{cases}
+1  & d(x,x') \le r,\\[4pt]
-1 & d(x,x') > R~.
\end{cases}
\label{eq:dx}
\end{equation}
Points whose distance from $x$ lies in the margin region $(r,R]$ are left unlabeled, and the data distribution is assumed to assign zero probability to this region.

This definition applies to arbitrary metric spaces and relies only on the distance structure. To illustrate how it relates to classical linear classification with margin, consider the special case where~$\X$ is the unit $\ell_2$-sphere in $\mathbb{R}^d$, scaled to have diameter $1$. In this setting, the level sets of the distance function $d(x,\cdot)$ correspond to intersections of affine hyperplanes with the sphere. In particular, a linear classifier with margin $\gamma$ induces a distance-based labeling of the form above, with parameters $r = \tfrac{1}{2}-\gamma$ and $R = \tfrac{1}{2}+\gamma$. Thus, ~\Cref{eq:dx} corresponds to the margin structure underlying standard linear classification, while extending naturally to general metric spaces.


\begin{figure}[ht]
\caption{Distance based labeling on a sphere}
\label{fig:sphere-margin}
\centering
\begin{tikzpicture}[scale=2]

    \def\angle{20} 
    \def\hpwidth{0.4}
    \def\hpslant{2.1}
    
    \begin{scope}[rotate around={\angle:(0,0)}]
        \fill[gray!40,opacity=1] (-\hpslant,-\hpwidth) -- (1.5,-\hpwidth) -- (\hpslant,\hpwidth) -- (-1.5,\hpwidth) -- cycle;
        \path[name path=edge3] (\hpslant,\hpwidth) -- (-1.5,\hpwidth);
        \path[name path=circle] (0,0) circle (1);
        \path [name intersections={of=edge3 and circle,by={i1,i2}}];
        \draw (-\hpslant,-\hpwidth) -- (1.5,-\hpwidth);
        \draw (1.5,-\hpwidth) -- (\hpslant,\hpwidth);
        \draw (\hpslant,\hpwidth) -- (i1);
        \draw (i2) -- (-1.5,\hpwidth);
        \draw (-1.5,\hpwidth) -- (-\hpslant,-\hpwidth);
    \end{scope}

    \shade[ball color=magenta!20,opacity=0.3] (0,0) circle (1);

    \begin{scope}[rotate around={\angle:(0,0)}]
        \clip (-1,-1) rectangle (1,0);
        \fill[gray!40,opacity=1] (-\hpslant,-\hpwidth) -- (1.5,-\hpwidth) -- (\hpslant,\hpwidth) -- (-1.5,\hpwidth) -- cycle;
    \end{scope}

    \begin{scope}[rotate around={\angle:(0,0)}]
        \fill[gray!40,opacity=1] (0,0) ellipse (1 and 0.25); 
        \draw[red!60,dashed] (0,0) ellipse (1 and 0.25); 
    \end{scope}

    \begin{scope}[rotate around={\angle:(0,0)}]
        \draw[red!60] (0,0.11) ellipse (0.98 and 0.25);
        \draw[red!60,opacity=0.5] (0,-0.11) ellipse (0.98 and 0.25);
        \draw[->] (-0.5,-0.21) -- (-0.5,-0.31);
        \draw[->] (-0.5,-0.21) -- (-0.5,-0.11);
        \node[left, rotate=\angle] at (-0.5,-0.21) {$\gamma$};
    \end{scope}

    \begin{scope}[rotate around={\angle:(0,0)}]
        \draw[dashed] (0,0) -- (0,1);
        \fill (0,1) circle(0.02);
        \fill (0,0) circle(0.02);
        \node[above, rotate=\angle] at (0,1) {$x$};
        \node[above, rotate=\angle] at (-0.4,-0.12) {$x_1$};
        \node[below, rotate=\angle] at (-0.4,-0.34) {$x_2$};
        \fill (-0.4,-0.12) circle(0.02);
        \fill (-0.4,-0.34) circle(0.02);
    \end{scope}

    \node[below] at (0, -1)
  {$\begin{aligned}
    d(x,x_1)&=r\\
    d(x,x_2)&=R
  \end{aligned}$};

\end{tikzpicture}
\end{figure}
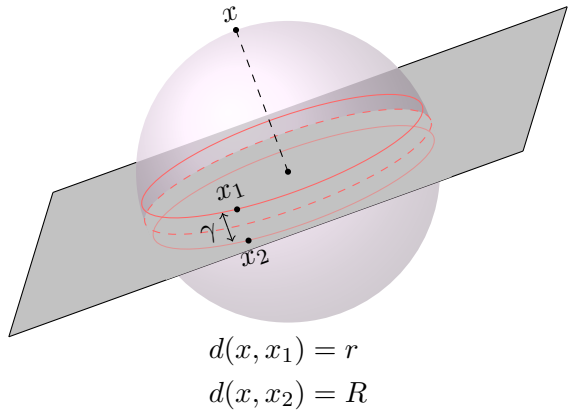

This simple concept class exhibits a sharp threshold behavior: for fixed values $r,R$ such that $R \geq 3r$, the class is learnable regardless of the underlying metric space. In this case, the $\vc$ dimension of the class is $1$ and even a single pair of points cannot be shattered by it\footnote{Note that the VC dimension continues to characterize PAC learnability for such classes of \emph{partial concepts}; this follows from known extensions of PAC learning theory, which we discuss in more detail later}. Indeed, assume by contradiction that there exists a pair $x, x' \in \X$ that are shattered.
Then, $\exists x_1 \in \X$ such that $d(x_1,x) \leq r$ and $d(x_1,x') \leq r$. By the triangle inequality, we have $d(x, x') \leq 2r$. On the other hand $\exists x_2 \in \X$ such that $d(x_2,x) < r$ and $d(x_2,x') > R \geq 3r$. By the triangle inequality, we get $d(x, x') > 2r$, a contradiction.

On the other hand, when $R < 3r$, the simple triangle-inequality-based argument no longer applies, and the behavior changes qualitatively: depending on the underlying metric space, the class may be arbitrarily hard to learn, or even unlearnable.
Indeed, let $X = A \cup B$ such that $A = \{ a_n \}_{n \in \mathbb{N}}$ is a countable set of points and $B = \{{b_s}\}_{\{ s \subseteq A \mid |s| < \infty \}}$ is a set of points referencing finite subsets of A.
    \[
    \rho(a_i, a_j) = \frac{r+R}{2}, \quad
    \rho(b_s, b_w) = \frac{r+R}{2}, \quad
    \rho(a_i, b_s) = 
    \begin{cases}
    r &  a_i \in s \\
    R & a_i \notin s\\
    \end{cases}
    \]

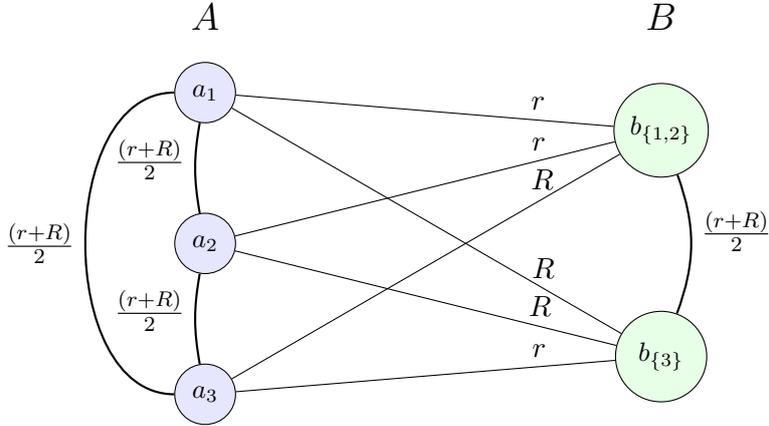
\begin{figure}[ht]
\caption{Unlearnable metric space for $R \leq 3r$}
\label{fig:metric-lin-shat}
\centering
\begin{tikzpicture}[
    node/.style={circle, draw, minimum size=8mm, font=\small, fill=blue!10},
    nodeB/.style={circle, draw, minimum size=8mm, font=\small, fill=green!10},
    edge/.style={-latex, thick},
    ]
\node[node] (a1) at (0,4) {$a_1$};
\node[node] (a2) at (0,2) {$a_2$};
\node[node] (a3) at (0,0) {$a_3$};

\node[nodeB, minimum size=12mm] (b1) at (6,3.5) {$b_{\{1,2\}}$};
\node[nodeB, minimum size=12mm] (b2) at (6,0.5) {$b_{\{3\}}$};

\node[font=\Large\bfseries] at ($(a1)+(0,1)$) {\bfseries $A$};
\node[font=\Large\bfseries] at ($(b1)+(0,1.5)$) {\bfseries $B$};

\draw (b1) to node[pos=0.2, above] {$r$} (a1);
\draw (b1) to node[pos=0.2, above] {$r$} (a2);
\draw (b1) to node[pos=0.2, above] {$R$} (a3);

\draw (b2) to node[pos=0.2, above] {$R$} (a1);
\draw (b2) to node[pos=0.2, above] {$R$} (a2);
\draw (b2) to node[pos=0.2, above] {$r$} (a3);

\draw[bend left=20, thick] (b1) to node[pos=0.4, right] {$\frac{(r+R)}{2}$} (b2);

\draw[bend right=10, thick] (a1) to node[pos=0.4, left] {$\frac{(r+R)}{2}$} (a2);
\draw[bend right=90, thick] (a1) to node[pos=0.5, left] {$\frac{(r+R)}{2}$} (a3);
\draw[bend right=10, thick] (a2) to node[pos=0.4, left] {$\frac{(r+R)}{2}$} (a3);


\end{tikzpicture}
\end{figure}

For any $R \leq 3r$ this is indeed a metric space, while any finite subset of $A$ is shattered by concepts defined by points in $B$. Hence the VC dimension of the class is infinite yielding an unlearnable setting.

Simply put, for sufficiently large margins, the learnability of these abstract margin-based classifiers relies only on the triangle inequality and not on any additional linear or geometric structure. This threshold is sharp: for every margin value below it, there exist metric spaces in which the class is not learnable. Thus, while learnability is guaranteed uniformly above the threshold, below it the behavior becomes space-dependent and may range from learnability to complete non-learnability.

In \Cref{sec:metric} we extend these ideas to a much richer concept class and identify total boundedness as a key structural
property in this setting.  When the margin is below the critical threshold, learnability remains possible provided
the underlying metric space is totally bounded\footnote{Recall that a space is totally bounded if for any $\gamma>0$ it can be covered by finitely many balls of radius $\gamma$.}.
Beyond the specific distance-based classes considered here, we
show that total boundedness in fact precisely characterizes the learnability of the entire class of Lipschitz functions over the metric space.


\paragraph{Margin-Based Learning in Linear Spaces.} 
Since our analysis of margin-based learning in general metric spaces provides only a partial characterization of margin based learnability, we shift our focus to the more structured setting of linear spaces. A well studied case of margin learning is the one of linear functionals on Hilbert spaces. Although infinite dimensional Hilbert spaces are not totally bounded, linear classification with margin in these spaces provide a canonical example for margin-based learning, a fact that cannot be accounted for by the metric space perspective.
This motivates our study of linear spaces, where additional geometric structure plays a central role.

At the same time, linear spaces arise naturally even when dealing with abstract, non-linear learning problems. A central paradigm in learning theory is to embed a classification problem into a linear space, most prominently via kernel methods, so that hypotheses are realized as linear functionals and margin-based generalization guarantees apply. This perspective raises the question of whether such linear embeddings merely provide a convenient analytical framework, or whether they fundamentally characterize margin-based learnability.

{These considerations motivate a more focused study of margin problems in the context of Banach spaces, in terms of their role as a potential universal model for margin-based learning via embeddings. We therefore ask the following question:}

\begin{center}
\textbf{To what extent is linear classification with margin universal in the context of margin-based learning? For example, can every learnable margin-based class be realized as linear functionals over an appropriate Banach space?}
\end{center}
We answer this question negatively and show that margin-based learnability does not, in general, reduce to linear embeddings into Banach spaces. To this end, we construct learnable margin-based classes with learning rates that violate a key property of the sample complexity bounds of learning linear classifiers in Banach spaces: that they must be polynomial in the inverse margin $\frac{1}{\gamma}$. This implies that the learnability of these classes cannot be reduced to learning them linearly in any Banach space.

\section{Background \& Problem Setup}

We begin by introducing several standard notions and the model of $\gamma$-learnability which is the main focus of investigation here.
Throughout we consider a fixed domain $\X$ and a function class $\F$ of real valued functions from $\X$ to $\mathbb{R}$. 

\paragraph{$\gamma$-learnability.} For any real-valued function $f:\X\to \mathbb{R}$, a binary labeled dataset 
$S=\{(x_i,y_i)\}_{i=1}^n$, where $x_i\in \X$ and $y_i \in \{-1,1\}$, is said to be 
\textbf{$\gamma$-realized} by $f$ if \(f(x_i)\cdot y_i > \gamma\), for all $(x_i,y_i)\in S$.
A distribution $\D$ is \textbf{$\gamma$-realizable} by $\F$ if there exists an $f \in \F$ that $\gamma$-realizes almost surely any i.i.d. dataset sampled from $\D$. 
A function class $\F$ is said to be \textbf{$\gamma$-learnable} if there exists a learning algorithm $\A$, mapping a sample $S$ to a function $\A(S):\X\to \{\pm1\}$,
and a sample-complexity bound
$m_{\F_{\gamma}} : (0,1) \times (0,1) \to \mathbb{N}$
such that for every $\varepsilon,\delta \in (0,1)$ and every distribution $\D$ that is $\gamma$-realizable by $\F$, the following holds:
If $S \sim \D^m$ with
$m \ge m_{\F,\gamma}(\varepsilon,\delta)$
then with probability at least $1-\delta$ over the draw of $S$,
\[
\Pr_{(x,y)\sim\D}
\bigl[
(\A(S)(x)) \cdot y \neq 1
\bigr]
\le \eps
\]

\paragraph{Partial concept classes}
We formulate the notion of margin concept class by using partial concept classes, as defined in
\cite{Long01,Partial}. Given an input space $\X$, a partial concept is a map $h:\X\to \{-1,1,\star\}$, where we think of $h(x)=\star$ as meaning that $h$ is undefined on $x$. A partial concept class $\H$ is a collection of partial concepts. A sample $S=\{(x_i,y_i)\}_{i=1}^n$ is realizable by the partial concept $\H$, if there is some $h\in \H$ such that $h(x_i)=y_i\neq \star$ for all $1\leq i\leq n$. A distribution $\D$ over $\X\times \{-1,1\}$ is realizable by $\H$ if any sample drawn from $\D$ is realizable with probability $1$. A set $\{x_i\}_{i=1}^n$, is said to be shattered by $\H$ if for any $y\in\{-1,1\}^n$ the labeled sample $\{(x_i,y_i)\}_{i=1}^n$ is realizable by $\H$, and $\vc(\H)$, the VC dimension of $\H$ is the maximal size of a shattered set (or infinite if there are shattered sets of arbitrary size).

A Partial concept class $\H$ is called learnable if there exists a learning algorithm $\A$ and a sample-complexity bound
$m_{\H} : (0,1) \times (0,1) \to \mathbb{N}$
such that for every $\varepsilon,\delta \in (0,1)$ and every realizable distribution $\D$  the following holds:
If $S \sim \D^m$ with
$m \ge m_{\H}(\varepsilon,\delta)$
then with probability at least $1-\delta$ over the draw of $S$,
\[
\Pr_{(x,y)\sim\D}
\bigl[
(\A(S)(x)) \neq y 
\bigr]
\le \eps
\]


A fundamental result in the theory of PAC learning is that the VC dimension governs the sample complexity of a class. Indeed the results by \cite{Partial, AdenAli2023} imply the following characterization 
\begin{customthm}{(Characterization of optimal Sample complexity via VC dimension)}
      Let $\H\subset\{-1,1,\star\}^\X$ be some partial concept class. Then $\H$ is learnable if and only if $\vc(\H)$ is finite, in which case we have the following bound on its optimal sample complexity 
    \[
    m_\H(\varepsilon,\delta)=\Theta(\frac{\vc(\H)+\log \frac{1}{\delta}}{\varepsilon}).
    \]
\end{customthm}

For a $\gamma>0$, and a function $f:\X\to \R$, the margin classifier defined by $f$ is the partial concept class $h_f^\gamma:\X\to \{-1,1-\star\}$ defined by 
\[
h_f^{\gamma}(x)=\begin{cases}
    +1 \quad &f(x)\geq\gamma,
    \\-1\quad &f(x)\leq-\gamma,
    \\\star\quad &f(x)\in(-\gamma,\gamma).
\end{cases}
\]
And given a set of functions $\F\subset \R^\X$, the induced partial concept class of classifiers is 
$\H_\F^\gamma=\{h_f\;:\; f\in \F\}$.
By definition, the margin problem $\F$ is $\gamma$-learnable if and only if $\H_\F^\gamma$ is learnable, and with the same sample complexity.

\paragraph{$\gamma$-fat shattering dimension.}
We recall the classic definition of the fat-shattering dimension as in~\cite{kearns1994efficient,alon1997scale}.
A set $X\subset \X$ is $\gamma$-shattered by a function class $\F$ if there exists a witness threshold function 
$r:X\to \R$ such that for any labeling $b:X\to\{-1,1\}$ there exists a function $f_b\in \F$ satisfying:
\[
f_b(x) \geq r(x)+\gamma \quad \text{if } b(x)=1, \quad \text{and} \quad f_b(x) \leq r(x)-\gamma \quad \text{if } b(x)=-1,
\]
for all $x\in X$.
The \textbf{$\gamma$-fat shattering dimension} of $\F$, denoted by $\fat_\gamma(\F)$, is the cardinality of the largest $\gamma$-shattered set $X\subset \X$ (or infinite if there exist shattered sets of arbitrarily large size).

Several restrictive variants of $\fat_\gamma(\F)$ exist in the literature. These include confining the threshold function $r$ to be a constant across $X$, strictly fixing it to 0 (denoted $\fat^0_\gamma(\F)$), and replacing the inequality with an equality $f_b(x) =  r(x) \pm \gamma$ (see, e.g. \cite{simon1997bounds,attias2023optimal,alon1997scale} for examples of these variants). The following technical result reconciles the numerous definitions and establishes that for many naturally structured function classes, these alternative formulations coincide.

\begin{prop}[Characterization of shattering in margin spaces]\label{Prop:Conv-Shatt}
Let $\X$ be a set, and let $\F$ be a convex and symmetric class of functions from $\X$ to $\R$; that is:
\begin{enumerate}
    \item $\F$ is closed under convex combinations.
    \item $f \in \F \implies -f \in \F$ .
\end{enumerate}
Let $\gamma>0$, and let
\(S=\{x_1,\ldots,x_n\}\subseteq \X\)
be a set of $n$ distinct points. Then the following conditions are equivalent:
\begin{enumerate}
    \item[\textrm(1)] $S$ is $\gamma$-shattered by $\F$: There exists a threshold function $r:S\to\R$ such that for every $s\in\{-1,1\}^n$, there exists $f\in\F$ satisfying
    \[
        s_i\bigl(f(x_i) - r(x_i)\bigr)\geq \gamma
        \qquad \forall 1\leq i\leq n .
    \]
    \item[\textrm(2)] $S$ is $\gamma$-shattered at 0 by $\F$ ($\gamma$-shattered with threshold $r=0$): For every $s\in\{-1,1\}^n$, there exists $f\in\F$ satisfying
    \[
        s_i f(x_i)\geq \gamma
        \qquad \text{for all } 1\leq i\leq n .
    \]

    \item[\textrm(3)] $\F$ contains an $n$-dimensional $\gamma$-cube on $S$; that is, for every $y\in[-\gamma,\gamma]^n$, there exists $f\in\F$ such that
    \[
        f(x_i)=y_i
        \qquad \text{for all } 1\leq i\leq n .
    \]
    
    \item[\textrm(4)] For every $\lambda\in\mathbb{R}^n$ satisfying $\sum_{i=1}^n |\lambda_i|=1$, there exists $f\in\F$ such that
    \[
        \sum_{i=1}^n \lambda_i f(x_i)\geq \gamma.
    \]

\end{enumerate}
Moreover, the equivalence of conditions \emph{(2)--(4)} holds even without the symmetry assumption on~$\F$.
\end{prop}

The proof is deferred to~\Cref{proof:Conv-Abs} while we turn our attention to the sample complexity of such classes. A consequence of \cite{Long01,Partial,AdenAli2023} is that the fat-shattering condition with threshold $r=0$ (condition $(2)$ of Proposition~\ref{Prop:Conv-Shatt}) precisely characterizes the optimal sample complexity of learning a real-valued function class. Combining this with Proposition~\ref{Prop:Conv-Shatt}, we conclude that the learnability of convex and symmetric function classes is characterized by $\fat_\gamma(\F)$.
\begin{corollary}\label{Cor:Conv-Complexity}
    Let $\F$ be a convex and symmetric class of functions from $\X$ to $\R$; that is:
    \begin{enumerate}
        \item $\F$ is closed under convex combinations.
        \item $f \in \F \rightarrow -f \in \F$.
    \end{enumerate} 
    Then for any $\gamma > 0$, the following conditions are equivalent \begin{enumerate}
        \item[\textrm(1)] $\F$ is $\gamma$-learnable.
        \item[\textrm(2)] $\fat_\gamma(\F) < \infty$.
        \item[\textrm(3)] $\fat^0_\gamma(\F) < \infty$.
    \end{enumerate}
Moreover, for $\fat_\gamma(\F)=d$, the optimal sample complexity for learning $\F$ with margin $\gamma$ satisfies:
\[
    m_{\F,\gamma}(\varepsilon,\delta)=\Theta\left(\frac{d+\log (1/\delta)}{\varepsilon}\right).
\]
\end{corollary}

A few comments regarding the conditions in Proposition \ref{Prop:Conv-Shatt} are in order. Condition~$(3)$ is a natural geometric property that connects the learnability of margin spaces to interpolation problems. It also closely relates to the $\gamma$-Natarajan dimension and $\gamma$-Graph dimension studied by \cite{attias2023optimal}, demonstrating that for the classes under discussion, both are equivalent to the fat-shattering dimension. Condition~$(4)$, while less intuitive in an abstract metric setting, turns out to be a useful tool whose geometric meaning becomes more transparent in linear structures. We defer a detailed discussion of this connection to~\ref{par:banach}, following our Banach space preliminaries.

\paragraph{Metric Spaces.}
Recall that a domain $\X$, together with a function $d:\X^2 \to \mathbb{R}$ is a metric space if~$d$ is symmetric, nonnegative with $d(x,y)=0$ if and only if $x=y$, and satisfies the traingular inequality, i.e., \( d(x,y) \le d(x,z)+d(z,y).\)
For a metric space $\X$, let $\Lip_{\X}$ (or simply $\Lip$) denote the class of all $1$-Lipschitz functions over the metric space. Namely, the set of all functions $f:\X\to \R$ satisfying \[\forall x,y\in X,\quad |f(x)-f(y)|\leq d(x,y).\] 
Within this class we consider the following  class of bounded linear combinations of point-distance functions:
\[
\Dx := \Biggl\{ 
\sum_{i=1}^\infty a_i\, d_{x_i} \ \Bigg|\ 
a_i \in \mathbb{R},\ 
x_i \in \X, \ 
\sum_{i=1}^\infty |a_i| \leq 1
\Biggr\},
\]
where each $d_{x_i}$ is defined by $d_{x_i}(x) := d(x_i,x)$.
These bounded combinations of distance functions
generalize the notion of hyperplanes beyond the unit sphere. For example, consider a subset of $\Dx$ of functions which are differences over $n=2$ distance functions, i.e. $f(x) = \frac{1}{2}\|x-x_1\| - \frac{1}{2}\|x-x_2\|$ for $x_1, x_2 \in \X$. Each such function encodes a half-space over the space $\X = \mathbb{R}^d$ while its 0-level set corresponds to a hyperplane in $\mathbb{R}^d$.

\paragraph{Banach Spaces.}
\label{par:banach}
A well studied instance of metric spaces are Banach spaces which are the basis for high-dimensional linear classification. Let $\X$ be a linear space, equipped with a norm $\|\cdot\|$. Recall that $\|\cdot\|$ is a norm if it is non-negative and positive-definite (i.e. $\|x\|=0$ iff $x=0$), absolutely homogeneous  (i.e.\ $\|\lambda x\|=|\lambda| \|x\|$ and satisfies the triangular inequality:
\[ \|x+y\|\le \|x\|+\|y\|.\]
The space $\X$ is called a Banach space if it is complete with respect to $\|\cdot\|$; that is, any Cauchy sequence $\{v_n\}_{n=1}^\infty\subseteq \X$ converges to a limit in $\X$.


Let $\B$ be a Banach space. In the context of linear classification with margin over $\B$, 
we assume that the domain is the unit ball $\B_1=\{x\in \B\;:\; \norm{x}\leq 1\}$, and that
the corresponding class $\F$ is the set of linear functionals whose \emph{dual-norm} is at most one. Recall that given a norm $\|\cdot\|$ on $\B$ the dual norm is a norm over~$\B^\star$, the linear functionals over $\B$, defined as
\( \|w\|= \sup_{\|x\|=1} w(x).\)
The class of $1$-bounded linear functions is then denote by 
$L_\B=\{w\in \B^\star\;
:\; \norm{w}_\star\leq 1\}$, where we consider each linear functional as a function over $\B_1$.

As noted in \Cref{Prop:Conv-Shatt}, Condition $(4)$ finds a natural realization in linear spaces where $\F = L_\B$.
In this setting, a straightforward application of the Hahn-Banach theorem yields 
\[
\gamma \leq \sup_{f\in \F}\sum \lambda_i f(x_i)=\sup_{f\in \F}f(\sum \lambda_i x_i)=\norm{\sum \lambda_i x_i}.
\]
This dual formulation yields two complementary perspectives. The first is through the geometric lens of the hyperplane separation theorem. A set $\{x_i\}_{i=1}^n$ is $\gamma$-shattered (with witness 0) if for every possible binary labeling $y = \{-1,1\}^n$, the convex hull of $\{y_ix_i\}_{i=1}^n$ is at least $\gamma$ far away from 0.
This mirrors the classical intuition that a shattered set must admit a linear margin from the origin for any arbitrary labeling \cite[see, e.g.,][]{gurvits2001note}.
The second perspective interprets shattering as a scale-sensitive generalization of linear independence.
While a set of vectors is linearly independent if every non-zero linear combination is non-trivial, Proposition~\ref{Prop:Conv-Shatt} implies that $\{x_i\}_{i=1}^n$ is $\gamma$-shattered\footnote{Recall that for a Banach space $\X$, we assume the hypothesis class consists of all linear functionals on $\X$ of norm at most one.} if and only if any (normalized) linear combination is $\gamma$-far from zero. 
We also note that in this case, the linear map $T:\ell_1^n\to \X$ defined by $Te_i=x_i$ shows the existence of a $\gamma$-isomorphic copy of $\ell_1^n$ in $\X$, thus relating fat-shattering to notions in local theory of Banach spaces as observed by \cite{mendelson2002learnability,mendelson2004shattering}. 


\section{Related Work}
\label{sec:related_work}

Our work intersects several foundational streams in statistical learning theory: dimension-independent generalization in normed spaces, metric-space classification, scale-sensitive dimensions, and the limits of kernel-type linear embeddings. Below, we position our contributions within this landscape.

\subsection{Large-Margin Generalization in Linear and Banach Spaces}
The principle of margin-based generalization is a cornerstone of classical learning theory, providing a mechanism for dimension-independent performance in highly over-parameterized models. In Euclidean and Hilbert spaces, the generalization error of linear classifiers can be bounded purely in terms of the margin, bypassing the nominal dimensionality of the space \cite{Vapnik1998, bartlett1998sample, Bartlett2002,cortes1995support}. 

The generalization of this paradigm to more abstract normed and Banach spaces has been   explored through the lens of structural functional analysis. Notably, \cite{gurvits2001note} investigated scale-sensitive dimensions for linear functionals, illustrating how margin requirements control capacity via the geometric properties of the underlying space. Building on this, \cite{mendelson2002learnability, mendelson2004shattering} established profound connections between the fat-shattering dimension of linear functionals on a Banach space and its structural properties such as Rademacher type and cotype. Our work directly builds on these insights in our second main result, leveraging Mendelson's asymptotic capacity bounds to prove the impossibility of universally embedding abstract metric margin spaces into learnable Banach structures.

\subsection{The Lipschitz Paradigm in Metric Spaces and its Bottlenecks}
Moving beyond vector spaces, classical extensions of margin learning to abstract metric spaces have predominantly adopted the framework introduced by \cite{Luxburg2004}. In their setting, the hypothesis class is defined as the \emph{entire class of bounded Lipschitz functions}, where the inverse of the Lipschitz constant behaves as the natural metric analog to a linear margin. 

However, a fundamental bottleneck of this approach is that the full Lipschitz class is exceptionally expressive. As studied in \cite{Luxburg2004, Aryeh2014} and extended below, there is a direct connection tying the dimensionality (total boundedness) of a space to the uniform learnability of the class of Lipschitz functions defined on it. 
This requirement represents a severe restriction; indeed, learning the full Lipschitz class is hard even in standard infinite-dimensional Hilbert or Banach spaces because
the space lacks total boundedness, restricting analysis to highly structured domains, such as metric
spaces with a low doubling dimension.   

\subsection{Generalizations of Linear Classifiers to Metric Spaces}
In light of the hardness of learning the full Lipschitz class, an independent and highly active line of inquiry has sought to identify and learn restricted sub-classes that act as true structural analogs to linear functions on metric domains. Prior attempts to generalize linear classification to metric spaces predominantly include
Kernel Methods and implicit linearization \cite{scholkopf2002learning}, 
learning with general similarity functions \cite{balcan2008improved} and nearest-prototype decision boundaries \cite{anthony2018large}.

Rather than targeting the full, unlearnable Lipschitz class or relying on implicit algebraic structures, our work focuses precisely on a well-defined structural generalization of hyperplanes.
Our contribution reveals a fundamentally different distribution-free learning landscape than prior work, bypassing requirements such as total boundedness, doubling dimensions, or similarity-goodness axioms.

\subsection{Scale-Sensitive Dimensions and Partial Concepts}
Methodologically, our distribution-free analysis builds on the fundamental characterization of real-valued learnability established by \cite{kearns1994efficient, alon1997scale,BARTLETT1998}. To model margin-based classifiers over abstract spaces rigorously, we formalize this setting using the framework of partial concept classes \cite{Long01,Partial} and leverage the modern paradigm of optimal sample complexity bounds for partial concepts established by \cite{AdenAli2023}.

\section{Main Results}
\subsection{Learnability in Metric Spaces}\label{sec:metric}

Let $\X$ be a metric space. We begin by studying the learnability of the class $\Dx$ of bounded linear combinations of distance functions, which, as noted earlier, generalizes the notion of half-spaces to general metric spaces.
We show that the class $\Dx$ fulfills a threshold-type phenomenon similar to the one exhibited in the introductory example. When the margin $\gamma$ is sufficiently large, $\Dx$ is $\gamma$-learnable in every metric space - hinging only on the triangle inequality - while smaller margins admit the full spectrum of possibilities, including non-learnability.

\begin{boxtheorem}{A Dichotomy for Distance Functions.}{Learn-Dx}
Let $\X$ be a metric space and let $\Dx$ denote the class
\[
\Dx := \Biggl\{ 
\sum_{i=1}^\infty a_i\, d_{x_i} \ \Bigg|\ 
a_i \in \mathbb{R},\ 
x_i \in \X, \ 
\sum_{i=1}^\infty |a_i| \leq 1
\Biggr\},
\]
where $d_x(\cdot) := d(\cdot,x)$.
Then the following hold:
\begin{enumerate}
    \item For every $\gamma \ge \tfrac{1}{3}$ and every space $\X$ with $\mathtt{diam}(\X)\leq1$, the class $\D_{\X}$ is $\gamma$-learnable.
    \item For every $\gamma < \tfrac{1}{3}$, there exists a space $\X$ with $\mathtt{diam}(\X)\leq1$ such that $\D_{\X}$ is not $\gamma$-learnable.
\end{enumerate}
\end{boxtheorem}
By rescaling the metric, the threshold $\tfrac{1}{3}$ corresponds to a normalized margin of $\gamma / \mathtt{diam}(\X)$.

The proof is similar to the one from the introduction which shows the dichotomy of the class from ~\Cref{eq:dx}, both in how learnability follows from the triangle inequality, and in the construction of the non learnable class. For the full proof see \Cref{proof:Learn-Dx}. 

The proof of the lower bound relies on constructing a metric space that is not \emph{totally bounded}. Recall that a metric space is totally bounded if, for every $\gamma>0$, there exists a finite cover of $\X$ by balls of radius~$\gamma$.
It is therefore natural to ask whether the dichotomy in the above theorem persists under a total boundedness assumption.
Our next result shows that if $\X$ is totally bounded, then $\Dx$ is $\gamma$-learnable for \emph{every} $\gamma>0$, and, moreover, shows that total boundedness is an exact characterization of Lipschitz learnability.
While the sufficiency of total boundedness can be inferred from existing sample complexity bounds \cite{Luxburg2004,Aryeh2014}, we show that it is also necessary, thereby establishing that total boundedness is equivalent to the learnability of Lipschitz functions.

\paragraph{Lipschitz Functions.}
Recall that a function $f:\X\to\mathbb{R}$ is \emph{1-Lipschitz} if $\lvert f(x)-f(y)\rvert \le d(x,y)$ for all $x,y\in\X$, and let $\Lip_{\X}$ denote the class of all such functions.
Note that $\Dx \subseteq \Lip_{\X}$, since every distance function is 1-Lipschitz.

\begin{boxtheorem}{Learnability of Lipschitz Functions.}{Lip}
Let $\X$ be a metric space and let $\Lip$ denote the class of all $1$-Lipschitz functions over $\X$.
Then, the following are equivalent.
\begin{enumerate}
    \item $\Lip$ is $\gamma$-learnable for every $\gamma>0$.
    \item $\X$ is totally bounded.
\end{enumerate}
Moreover, $\fat_{\gamma}(\Lip)$ is exactly the $2\gamma$-packing number of $\X$ (i.e., the maximum size of a subset of $\X$ with pairwise distances at least $2\gamma$).
\end{boxtheorem}
The main insight underlying the proof is that a dataset $S=\{(x_i,y_i)\}_{i=1}^n$ is $\gamma$-realizable by $\Lip$ if and only if the minimum distance between positively and negatively labeled examples is at least $2\gamma$.
The full proof is given in \Cref{Proof:Lip}.

Previous works mentioned have established learnability guarantees for Lipschitz function classes in terms of covering numbers of the underlying metric space.
In particular, \cite{Luxburg2004,Aryeh2014} derived sample complexity bounds for learning $\Lip_{\X}$ that depend on the covering numbers of~$\X$.
Since finite covering numbers imply that $\X$ is totally bounded, these works show that total boundedness is a sufficient condition for learnability.
Our contribution therefore is showing that total boundedness is not merely sufficient but also necessary. That is, total boundedness exactly characterizes when the class $\Lip_{\X}$ is learnable under a margin assumption.
Moreover, our characterization via packing numbers yields tight bounds on the optimal sample complexity.

\subsection{Learnability in Banach spaces}
We now turn to the question of the universality of linear classification with margin.
Linear classifiers play a central role in learning theory, both due to their favorable sample complexity properties and because many abstract learning problems are approached via linear embeddings, for instance through kernel methods.

A notable feature of linear classification is that, unlike the general metric-space setting, margin-based learning in Hilbert spaces is possible for every $\gamma>0$, despite the fact that such spaces are infinite-dimensional and not totally bounded.
This naturally leads to the question of whether general margin-based learning problems can always be reduced to linear classification in normed spaces.
More precisely, we ask whether every $\gamma$-learnable margin-based concept class can be embedded - possibly via an appropriate kernel - into a Banach space in which linear classification with margin is learnable.

We show that this universality does not hold in general: there exist concept classes that are learnable for every $\gamma>0$, yet cannot be embedded into any Banach space admitting learnable linear classification with margin. Our argument relies on a characterization of the possible asymptotic dependence of the sample complexity on the margin in Banach spaces, given in the following result by~\cite{mendelson2002learnability,mendelson2004shattering}:

\begin{theorem}[Fat-Shattering Dimension Bound for Linear Classifiers]
\cite{mendelson2002learnability,mendelson2004shattering}  \\
\label{thm:Banch-bound}
     \noindent Let $\B$ be a Banach space and assume that $\fat_\gamma(L_\B)<\infty$ for some $\gamma\in(0,1)$. Then $\fat_\gamma(L_\B)<\infty$ for all $\gamma\in (0,1)$ and there is some $p\geq 2$ such that 
     \[
     \fat_\gamma(L_\B)=O(\frac{1}{\gamma^p}).
     \]
\end{theorem}

It is interesting to note that the converse is also true. That is, for every $p\geq 2$ there exists a Banach space for which $\fat_\gamma(L_\B)=\Theta(\frac{1}{\gamma^p})$. Combining the two, we get a taxonomy of Banach spaces based on the learnability rates of their linear functionals with margins.

\subsubsection{Embedding margin problems into Banach spaces}

Embedding learning problems into linear spaces is a classical technique that has played a particularly influential role in machine learning, most notably through kernel methods. In this framework, nonlinear concept classes are analyzed via linear predictors in a suitable feature space.
Given the central role of kernel methods in explaining margin-based generalization, it is natural to ask to what extent margin-based learnability can - or must - be understood through such linear embeddings.

To formalize this idea, we define an embedding of a margin problem $\F$ over a domain $\X$ into a Banach space $\B$ as a pair of maps $\Phi : \X \to \B_1$ and $\Psi : \F \to \B_1^\star$, where $\B_1$ and $\B_1^\star$ denote the closed unit balls of $\B$ and its dual $\B^\star$, respectively. We require that there exist constants $c_1,c_2>0$ such that for all $x\in \X$ and $f\in \F$,
\[
c_1\,\Psi(f)\big(\Phi(x)\big) \;\le\; f(x) \;\le\; c_2\,\Psi(f)\big(\Phi(x)\big).
\]
When $c_1=c_2=1$, this condition reduces to $\Psi(f)(\Phi(x)) = f(x)$, i.e., the class $\F$ is realized exactly as the unit ball of linear functionals on some Banach space $\B$. Allowing $c_1$ and $c_2$ to differ relaxes this requirement, permitting controlled distortion when representing $\F$ as a family of linear functionals.

The guiding principle behind this definition is that of reduction of learning problems: an embedding of a margin problem $\F$ into a Banach space $\B$ reduces the task of learning $\F$ to that of learning linear classifiers in~$\B$. Indeed, if a labeled sample $\{(x_i,y_i)\}_{i=1}^n$ is $\gamma$-realizable by $\F$, then the transformed sample $\{(\Phi(x_i),y_i)\}_{i=1}^n$ is $(C\gamma)$-realizable in $\B$, where $C = 1/\max(c_1,c_2)$, so 
\[
\fat_\gamma(\F)\leq \dim_{C\gamma}(L_\B).
\] 
From Theorem \ref{thm:Banch-bound} 
 we know that if a Banach space is $\gamma$-learnable for some $\gamma\in(0,1)$, then it is learnable for all $\gamma\in(0,1)$. Thus, if $\F$ admits an embedding into a learnable (for any or all $\gamma$) Banach space $\B$, then $\F$ itself is $\gamma$-learnable for all margins $\gamma\in (0,1)$.

This naturally raises the converse question: \emph{can every margin problem $\F$ that is learnable for all $\gamma\in(0,1)$ be embedded into a margin-learnable Banach space?} Unfortunately, the answer is negative, as shown by the following result.

\begin{boxtheorem}{Learnable class with no learnable Banach embedding}{embed}
There exists a symmetric and convex set of functions $\F \subset \R^{\N}$ that is $\gamma$-learnable for all $\gamma\in(0,1)$, yet admits no embedding into any learnable Banach space.
\end{boxtheorem}


The full proof of theorem~\ref{Thm:embed}, is given in \ref{proof:embed}. The main argument is an application of  theorem~\ref{thm:Banch-bound}, with a construction of a set of functions $\F$ for which $\dim_\F(\gamma)$, is not polynomial. Specifically, given a decreasing sequence of positive numbers $\{a_n\}_{n=1}^\infty$, define a class $\F$ of functions on the natural numbers by
\[
\F=\bigl\{f\in (0,1)^\N \;:\; |f(n)|\leq a_n\bigr\}.
\]
Clearly $\F$ is symmetric and convex. Additionally, $\fat_\gamma(\F)$ is the largest integer $n$ for which $a_n<\gamma$. Consequently, if $a_n$ decreases to $0$ slowly enough we have that $\fat_\gamma(\F)$ is not bound by any polynomial of $\frac{1}{\gamma}$. Hence $\F$ cannot be embedded into any learnable Banach space $\B$, since by theorem \ref{thm:Banch-bound}, any such embedding would enforce a polynomial upper bound on $\fat_\gamma(\F)$.

\section{Proofs}

\subsection{Proof of learnability of metric classes}

\begin{proof}[Proof of \Cref{Thm:Learn-Dx}]\label{proof:Learn-Dx}
    First, notice that $\Dx$ is symmetric and closed under convex combinations. Hence by~\Cref{Cor:Conv-Complexity} it suffices to show that $\fat^0_\gamma(\Dx) < \infty$.
    Observe that $\Dx$ can be partitioned into two classes
    $\Dx = \Dx^{<}\uplus\Dx^{>}$ such that
    \[
    \Dx^{>} := \Bigl\{ \sum_{i=1}^n a_i\, d_{x_i} \in \Dx\mid \sum_{i:a_i \geq 0} a_i \ge \tfrac12 \Bigr\}, \quad
    \Dx^{<} := \Bigl\{ \sum_{i=1}^n a_i\, d_{x_i} \in \Dx\mid \sum_{i:a_i \ge 0} a_i < \tfrac12 \Bigr\}.
    \]
    Since the property $\fat^0_\gamma < \infty$ is closed under finite unions (a direct consequence of the scale-sensitive Sauer–Shelah lemma), it suffices to show that both $\Dx^{>}$ and $\Dx^{<}$ have finite $\gamma$-fat-shattering dimension at $0$. We show this for $\Dx^{>}$. By the symmetric nature of the partition, the bound for $\Dx^{<}$ follows immediately from that of $\Dx^{>}$. Specifically, we show that if $\fat^0_\gamma(\Dx^{\gamma}) \geq 2$ then $\gamma\leq \frac{1}{3}$.
    
    Assume there are some $x, x' \in \X$, which are $\gamma$-shattered at 0 by ${\Dx}^{>}$. 
    Then there is some $f \in \Dx^{>}$ such that $f(x) \leq -\gamma$ and $f(x') \leq -\gamma$. By definition, there are points $\{x_i\}_{i=1}^n \subseteq \X$ and scalars $\{a_i\}_{i=1}^n \subseteq \mathbb{R}$, with $\sum |a_i|\leq1$, and $\sum_{i:a_i>0}a_i\geq \frac{1}{2}$  such that  $f(x) = \sum_{i=1}^n a_i d_{x_i}(x) \leq -\gamma$ and $f(x') = \sum_{i=1}^n a_i d_{x_i}(x') \leq -\gamma$. Denote $I=\{1\leq i\leq n \;:\; a_i\geq 0\}$, and $J=\{1\leq j\leq n \;:\; a_j<0\}$ so we have 
    \[
    \sum_{i\in I}a_id(x,x_i)-\sum_{j\in J}|a_j|d(x,x_j)=\sum_{k=1}^n a_kd(x,x_k)\leq -\gamma.
    \]
Which implies  
\[
 \sum_{i\in I}a_id(x,x_i)\leq \sum_{j\in J}|a_j|d(x,x_j)-\gamma.
\]
    And the same holds if we replace $x$ with $x'$. Together with $\sum_{i\in I}a_i\ge \frac{1}{2}$ and the triangle inequality, this implies
    \[
    \begin{aligned}
    \frac{1}{2}d(x,x') &\le \sum_{i\in I} a_i\, d(x,x') \\
            &\le \sum_{i\in I} a_i \bigl[d(x,x_i) + d(x_i, x')\bigr] \\
            &\leq \Bigl(\sum_{j\in J} |a_j| \bigl[d(x,x_j) + d(x_j, x')\bigr] \Bigr)- 2\gamma \\
            &\le \frac{1}{2}(1+1)-2\gamma=1-2\gamma. \\
    \end{aligned}
    \]
    Where in the last inequality we used the assumption that $\mathtt{Diam}(\X)\leq 1$.
    
    Additionally, from $\gamma$-fat shattering at 0, we have some 
    $\{x_i'\}_{i=1}^m \subseteq \X$ and $\{a_i'\}_{i=1}^m \subseteq \mathbb{R}$ such that $f'(x) = \sum_{i=1}^m a_i' d(x_i',x) > \gamma$ and $f'(x') = \sum_{i=1}^m a_i d(x_i',x') < -\gamma$. From this we get
    \[
    d(x,x') \geq \sum_{i=1}^{n} a_i' d(x,x')\geq \sum_{i=1}^na_i'[d(x,x_i')-d(x',x_i')]\geq 2\gamma.
    \]
    So our two inequalities gives \[
    2\gamma\leq d(x,x')\leq 2-4\gamma.
    \]
   Which implies $\gamma \leq \frac{1}{3}$. 
    Define a metric space $(\X, d)$ as follows: 
    Let $X = A \cup B$ such that $A = \{ a_n \}_{n \in \mathbb{N}}
    $ is a countable set of points and $B = B_1 \cup B_2$ where $ B_i = \{{b^i_s}\}_{\{ s \subseteq A \mid |s| < \infty \}}$ are a pair of sets of points referencing finite subsets of A.
    \[
    \rho(a_i, a_j) = \frac{2}{3}, \quad
    \rho(b^i_s, b^j_w) = \frac{2}{3}, \quad
    \rho(a_i, b^j_s) = 
    \begin{cases}
    \frac{2}{3} - \gamma &  a_i \in s ,\ j=1\\
    \frac{2}{3} + \gamma & a_i \notin s,\ j=1\\
    \frac{2}{3} + \gamma & a_i \in s,\ j=2\\
    \frac{2}{3} - \gamma & a_i \notin s,\ j=2\\
    \end{cases}
    \]
    For any $\gamma \le 1/3$ this is indeed a metric space with $Diam(\X) \le 1$. For any finite $A' \subset A$, let $\delta_{A'}:=\frac{1}{2}\cdot d_{b_{A'}^1} - \frac{1}{2} \cdot d_{b_{A'}^2} \in \Dx^{\gamma}$. Then $\delta\vert_{A} = -\gamma\cdot\mathbf{1}_{A'} + \gamma \cdot \mathbf{1}_{A\setminus A'}$. Therefore $\fat_\gamma(\Dx) > n$ for all $n$, concluding our proof.
\end{proof}

\begin{proof}[Proof of \Cref{Thm:Lip} - Learnability of Lipschitz functions]\label{Proof:Lip}
First, notice that $\Lip$ is symmetric and closed under convex combinations. Hence by~\Cref{Cor:Conv-Complexity}, $\gamma$-learnability of $\Lip$ is equivalent to the finiteness of $\fat^0_{\gamma}(\Lip)$. 

Let $S=\{(x_i,y_i)\}_{i=1}^n$, be some sample $\gamma$-realizable by $\Lip$, denote the positive and negative labeled examples by $S^+=\{(x,y)\in S\;:\; y=1\}$,
$S^-=\{(x,y)\in S\;:\; y=-1\}$, and note that $d(S^+,S^-)\geq2\gamma$. Conversely,  any such sample $S$ for which $d(S^+,S^-)\geq2\gamma$, is $\gamma$-realizable by $\Lip$. Indeed  
define $f\in \Lip$ by 
\[
f(x)=\frac{d(S^-,x)-d(S^+,x)}{2},
\]  
and note that if $x\in S^+$ then $f(x)=\frac{1}{2}\cdot d(S^-,x)\geq\gamma$, and similarly if $x\in S^-$ then $f(x)\leq-\gamma$. We can also see that $f\in \Lip$ since $d(A,y)\leq d(A,x)+d(x,y)$ for any set $A\subset \X$ and points $x,y\in \X$. Thus for any sample $S$:
\[
\textbf{S is $\gamma$-realizable by } \Lip \iff d(S^+,S^-)\geq2\gamma
\]

$1 \rightarrow 2$ Assume to the contrary that $\X$ is not totally bounded, i.e. that for every $n$, there exists a set $S \subseteq \X$ of cardinality $n$ such that $d(x_i,x_j)\geq2\gamma$ for any pair of different points $x_i,x_j \in S$. Then for any sign pattern $y\in\{\pm 1\}^n$ we have that $d(S^{-},S^{+})\geq2\gamma$, where $S^{-}=\{x_i\in S\;: y_i=-1\}$, and $S^+=\{x_i\in S\;:\; y_i=1\}$. Hence $\{(x_i,y_i)\}_{i=1}^n$ is $\gamma$-realized by $f\in\Lip$ given by \[
f(x)=\frac{d(S^-,x)-d(S^+,x)}{2}.
\]
Implying that $\fat^0_\gamma(\Lip) \geq n$ for all $n$.

$2 \rightarrow 1$ For the other direction, assume by negation that there exists $\gamma>0$ such that $\Lip$ is not $\gamma$-learnable. Let $S$ be a sample $\gamma$-shattered at 0 by $\Lip$. Then for any $x_i, x_j \in S$ there is some $f\in \Lip$ such that $f(x_i)\geq\gamma$, $f(x_j)\leq-\gamma$, hence we have \[
2\gamma\leq f(x_i)-f(x_j)\leq  d(x_i,x_j).
\]
Since this is true for sets $S$ of arbitrary size, we conclude that $\X$ is not totally bounded.
\end{proof}

\subsection{Proofs for margin embedding and characterization of shattering results}

\begin{proof}[Proof for \Cref{Thm:embed} - non embeddability of margin spaces]\label{proof:embed}
   Define $\F\subset (-1,1)^\N$ by 
    \[
    \F=\{f\in (-1,1)^\N\;: |f(n)|\leq \frac{1}{\log n} \}
    \]
    Clearly $\F$ is symmetric and convex. It is also clear that $S$ is $\gamma$-shattered by $\F$ if and only if $\gamma<\frac{1}{\log x}$ for all $x\in S$. Hence we deduce that $\fat_\gamma(\F)=\floor{e^{\frac{1}{\gamma}}}$, and in particular $\F$ is $\gamma$-learnable for all $\gamma$. Note that any embedding of $\F$ into a Banach space $\B$ will imply  that there is some constant $C>0$ such that 
\[
[e^{\frac{1}{\gamma}}]=\fat_\gamma(\F)\leq \fat_{C\gamma}(L_\B).
\]
    By Theorem \ref{thm:Banch-bound}, this implies that $B$ is not $\gamma$ learnable for any $\gamma>0$, as any learnable Banach space will obey the bound $\fat_\gamma(L_\B)=O(\frac{1}{\gamma^p})$ for some $p>0$.
\end{proof}

\begin{proof}[Proof of Proposition~\ref{Prop:Conv-Shatt} - Characterization of shattering in margin spaces]\label{proof:Conv-Abs}

    $(4)\implies (3): $ Assume toward contradiction that $(4)$ holds but $(3)$ does not hold, so there is some $y\in [-\gamma,\gamma]^n$ such that for every $f\in \F$ there is some $i \in [n]$ such that $f(x_i)\neq y_i$. Consider the set \[
    \F|_S=\{\big(f(x_1),f(x_2)\dots f(x_n)\big)\;:\; f\in \F\}.
    \] 
Since $\F$ is convex so is $\F|_S\subset \R^n$, and by assumption $y\notin \F|_S$. Hence by the Hahn-Banach theorem there is some vector $\lambda\in \R^n$, and scalar $a\in \R$ such that
\[
\sum_{i=1}^n \lambda_iy_i=a.
\]
But for any $f\in \F|_S$ we have 
\[
\sum_{i=1}^n \lambda_if_i<a.
\]
By normalization we may assume $\sum_{i=1}^n |\lambda_i|=1$. Hence, since we assumed $(3)$, we get that there is some $f\in \F$ such that 
\[
\sum_{i=1}^n \lambda_i f(x_i)\geq \gamma=\sum _{i=1}^n |\lambda_i|\gamma\geq \sum_{i=1}^n \lambda_i y_i =a
\]
In contradiction.

$(3)\implies (2) : $ Obvious.

$(2)\implies (4) :$ Let $\lambda\in \R^n$, be such that $\sum_{i=1}^n|\lambda_i|=1$.
Since we assume that $S$ is shattered by $\F$, for any sign pattern $s\in \{-1,1\}^n$, there is some $f\in \F$ such that $s_if(x_i)\geq \gamma$ for all $1\leq i\leq n$. In particular we may find such $f$ for the sign pattern $s_i=\sign (\lambda_i)$ (choosing arbitrarily in the cases where $\lambda_i=0$). Then for this $f$ we get for all $1\leq i\leq n$  
\[
\lambda_i f(x_i)=|\lambda_i|\sign(\lambda_i)f(x_i)\geq |\lambda_i|\gamma.
\]
From which we deduce 
\[
\sum_{i=1}^n \lambda_if(x_i)\geq \sum_{i=1}^n |\lambda_i|\gamma=\gamma.
\]
Which gives the desired result. 
Finally we show that if $(1)$ and $(2)$ are equivalent when $\F$ is also symmetric. Clearly $(2)\implies (1)$ so we only need to show the converse. Assume we have some $r:S\to \R$ such that for any $s\in \{-1,1\}^n$ there is $f_s$ such that $\big(f_s(x_i)-r(x_i)\big)s_i\geq \gamma$ for all $1\leq i\leq n$. Then for any such $s\in\{-1,1\}$ define 
\[
\phi_s=\frac{f_s-f_{-s}}{2}.
\]
Where $f_{-s}$ is the function corresponding the labeling $-s\in \{-1,1\}^n$.
Since $F$ symmetric we have that $-f_{-s}\in \F$, and since it is convex we have that $\phi_s\in \F$. To conclude the proof we note that for all $1\leq i\leq n$
\[
\phi_s(x_i)s_i=\frac{s_if_s-s_if_{-s}}{2}\geq \frac{\gamma+\gamma}{2}=\gamma.
\]
\end{proof}

\section*{Acknowledgments}
YA and RL are supported by the European Union (ERC, FoG 101116258).
SM and TW are supported by Israel PBC-VATAT, by the Technion Center for Machine Learning and Intelligent Systems (MLIS), and by the European Union (ERC, GENERALIZATION, 101039692).

Views and opinions expressed are those of the author(s) only and do not necessarily reflect those of the European Union or the European Research Council Executive Agency. Neither the European Union nor the granting authority can be held responsible for them.

We thank the anonymous reviewers for their insightful suggestions and comments regarding the fat-shattering dimension, which greatly improved the context and technical clarity of our results.

\bibliographystyle{plainnat}
\bibliography{bib}

@inproceedings{Long01,
  author       = {Philip M. Long},
  editor       = {David P. Helmbold and
                  Robert C. Williamson},
  title        = {On Agnostic Learning with \{0, *, 1\}-Valued and Real-Valued Hypotheses},
  booktitle    = {Computational Learning Theory, 14th Annual Conference on Computational
                  Learning Theory, {COLT} 2001 and 5th European Conference on Computational
                  Learning Theory, EuroCOLT 2001, Amsterdam, The Netherlands, July 16-19,
                  2001, Proceedings},
  series       = {Lecture Notes in Computer Science},
  volume       = {2111},
  pages        = {289--302},
  publisher    = {Springer},
  year         = {2001},
  url          = {https://doi.org/10.1007/3-540-44581-1\_19},
  doi          = {10.1007/3-540-44581-1\_19},
  bibsource    = {dblp computer science bibliography, https://dblp.org}
}

@ARTICLE{Aryeh2014,
  author={Gottlieb, Lee-Ad and Kontorovich, Aryeh and Krauthgamer, Robert},
  journal={IEEE Transactions on Information Theory}, 
  title={Efficient Classification for Metric Data}, 
  year={2014},
  volume={60},
  number={9},
  pages={5750-5759},
  doi={10.1109/TIT.2014.2339840}}

@article{Luxburg2004,
  author       = {Ulrike von Luxburg and
                  Olivier Bousquet},
  title        = {Distance-Based Classification with Lipschitz Functions},
  journal      = {J. Mach. Learn. Res.},
  volume       = {5},
  pages        = {669--695},
  year         = {2004},
  url          = {https://jmlr.org/papers/volume5/luxburg04b/luxburg04b.pdf},
  bibsource    = {dblp computer science bibliography, https://dblp.org}
}

@INPROCEEDINGS {Partial,
author = { Alon, Noga and Hanneke, Steve and Holzman, Ron and Moran, Shay },
booktitle = { 2021 IEEE 62nd Annual Symposium on Foundations of Computer Science (FOCS) },
title = {{ A Theory of PAC Learnability of Partial Concept Classes }},
year = {2022},
volume = {},
ISSN = {},
pages = {658-671},
doi = {10.1109/FOCS52979.2021.00070},
url = {https://doi.ieeecomputersociety.org/10.1109/FOCS52979.2021.00070},
publisher = {IEEE Computer Society},
address = {Los Alamitos, CA, USA},
month =Feb}

@article{Bartlett2002,
  author  = {Peter L. Bartlett and Shahar Mendelson},
  title   = {Rademacher and Gaussian Complexities: Risk Bounds and Structural Results},
  journal = {Journal of Machine Learning Research},
  year    = {2002},
  volume  = {3},
  pages   = {463--482},
}

@book{Vapnik1998,
  author    = {Vladimir N. Vapnik},
  title     = {Statistical Learning Theory},
  year      = {1998},
  publisher = {Wiley},
  address   = {New York},
}

@article{mendelson2004shattering,
  title={The shattering dimension of sets of linear functionals},
  author={Mendelson, Shahar and Schechtman, Gideon},
  year={2004}
}

@article{simon1997bounds,
  title={Bounds on the number of examples needed for learning functions},
  author={Simon, Hans Ulrich},
  journal={SIAM Journal on Computing},
  volume={26},
  number={3},
  pages={751--763},
  year={1997},
  publisher={SIAM}
}

@article{attias2023optimal,
  title={Optimal learners for realizable regression: Pac learning and online learning},
  author={Attias, Idan and Hanneke, Steve and Kalavasis, Alkis and Karbasi, Amin and Velegkas, Grigoris},
  journal={Advances in Neural Information Processing Systems},
  volume={36},
  pages={44707--44739},
  year={2023}
}

@article{alon1997scale,
  title={Scale-sensitive dimensions, uniform convergence, and learnability},
  author={Alon, Noga and Ben-David, Shai and Cesa-Bianchi, Nicolo and Haussler, David},
  journal={Journal of the ACM (JACM)},
  volume={44},
  number={4},
  pages={615--631},
  year={1997},
  publisher={ACM New York, NY, USA}
}

@article{cortes1995support,
  title={Support-vector networks},
  author={Cortes, Corinna and Vapnik, Vladimir},
  journal={Machine learning},
  volume={20},
  number={3},
  pages={273--297},
  year={1995},
  publisher={Springer}
}

@article{BARTLETT1998,
title = {Prediction, Learning, Uniform Convergence, and Scale-Sensitive Dimensions},
journal = {Journal of Computer and System Sciences},
volume = {56},
number = {2},
pages = {174-190},
year = {1998},
issn = {0022-0000},
doi = {https://doi.org/10.1006/jcss.1997.1557},
url = {https://www.sciencedirect.com/science/article/pii/S0022000097915579},
author = {Peter L. Bartlett and Philip M. Long}
}

@article{mendelson2002learnability,
  title={Learnability in Hilbert spaces with reproducing kernels},
  author={Mendelson, Shahar},
  journal={journal of complexity},
  volume={18},
  number={1},
  pages={152--170},
  year={2002},
  publisher={Elsevier}
}

@book{Cristianini2000,
  author    = {Nello Cristianini and John Shawe-Taylor},
  title     = {An Introduction to Support Vector Machines and Other Kernel-based Learning Methods},
  year      = {2000},
  publisher = {Cambridge University Press},
  address   = {Cambridge},
}

@inproceedings{AdenAli2023,
  title={Optimal PAC Bounds Without Uniform Convergence},
  author={Ishaq Aden-Ali and Yeshwanth Cherapanamjeri and Abhishek Shetty and Nikita Zhivotovskiy},
  booktitle={Proceedings of the 64th IEEE Annual Symposium on Foundations of Computer Science (FOCS)},
  year={2023}
}

@article{kearns1994efficient,
    title = {Efficient distribution-free learning of probabilistic concepts},
    journal = {Journal of Computer and System Sciences},
    volume = {48},
    number = {3},
    pages = {464-497},
    year = {1994},
    issn = {0022-0000},
    doi = {https://doi.org/10.1016/S0022-0000(05)80062-5},
    url = {https://www.sciencedirect.com/science/article/pii/S0022000005800625},
    author = {Michael J. Kearns and Robert E. Schapire}
}

@article{gurvits2001note,
  title={A note on a scale-sensitive dimension of linear bounded functionals in Banach spaces},
  author={Gurvits, Leonid},
  journal={Theoretical Computer Science},
  volume={261},
  number={1},
  pages={81--90},
  year={2001},
  publisher={Elsevier}
}

@book{scholkopf2002learning,
  title={Learning with kernels: support vector machines, regularization, optimization, and beyond},
  author={Sch{\"o}lkopf, Bernhard and Smola, Alexander J},
  year={2002},
  publisher={MIT press}
}

@article{balcan2008improved,
  title={Improved guarantees for learning via similarity functions},
  author={Balcan, Maria-Florina and Blum, Avrim and Srebro, Nathan},
  year={2008},
  publisher={Carnegie Mellon University}
}

@article{anthony2018large,
  title={Large width nearest prototype classification on general distance spaces},
  author={Anthony, Martin and Ratsaby, Joel},
  journal={Theoretical Computer Science},
  volume={738},
  pages={65--79},
  year={2018},
  publisher={Elsevier}
}

@article{bartlett1998sample,
  title={The sample complexity of pattern classification with neural networks: the size of the weights is more important than the size of the network},
  author={Bartlett, Peter L},
  journal={IEEE transactions on Information Theory},
  volume={44},
  number={2},
  pages={525--536},
  year={1998},
  publisher={IEEE}
}

\end{document}